\documentclass[transaction]{IEEEtran}
\usepackage{amsfonts}
\usepackage{amssymb}
\usepackage{hyperref}
\usepackage{graphicx}
\usepackage{enumerate}
\usepackage{amsmath}
\usepackage{color}
\usepackage{amsthm}
\usepackage{verbatim}
\usepackage{amsmath}
\usepackage{algorithm}
\usepackage{cite}
\usepackage{algpseudocode}
\usepackage{subfig}
\usepackage{graphicx}
\usepackage{float}
\usepackage{autobreak}
\usepackage{booktabs}  
\usepackage{array}   
\usepackage{xcolor}  
\usepackage{amsfonts,amssymb}
\providecommand{\keywords}[1]{\textbf{\textit{Index terms---}} #1}

\newcommand{\bp}{\begin{proof} \small }
\newcommand{\ep}{\end{proof} \normalsize}
\newcommand{\epx}{\end{proof} \small}
\newcommand{\bpa}{\begin{proofappx} \footnotesize }
\newcommand{\epa}{\end{proofappx} \small }
\newtheorem{theorem}{Theorem}

\newtheorem{assumption}{Assumption}

\newtheorem*{theorem*}{Theorem}
\newtheorem*{proposition*}{Proposition}
\newtheorem*{corollary*}{Corollary}
\newtheorem*{lemma*}{Lemma}
\newtheorem*{assumption*}{Assumption}
\newtheorem*{definition*}{Definition}
\newtheorem*{claim*}{Claim}

\newcommand{\be}{\begin{equation}}
\newcommand{\ee}{\end{equation}}
\newcommand{\bs}{\begin{subequations}}
\newcommand{\es}{\end{subequations}}
\newcommand{\bq}{\begin{eqnarray}}
\newcommand{\eq}{\end{eqnarray}}
\newcommand{\bqn}{\begin{eqnarray*}}
\newcommand{\eqn}{\end{eqnarray*}}

\newcommand{\ba}{\left[ \begin{array}}
\newcommand{\ea}{\\ \end{array} \right]}
\newcommand{\ben}{\begin{enumerate}}
\newcommand{\een}{\end{enumerate}}

\def\real{{\mathchoice%
{\hbox{\rm\setbox1=\hbox{I}\copy1\kern-.45\wd1 R}}
{\hbox{\rm\setbox1=\hbox{I}\copy1\kern-.45\wd1 R}}
{\hbox{\scriptsize\rm\setbox1=\hbox{I}\copy1\kern-.45\wd1 R}}
{\hbox{\scriptsize\rm\setbox1=\hbox{I}\copy1\kern-.45\wd1 R}}}}

\def\Zint{{\mathchoice{\setbox1=\hbox{\sf Z}\copy1\kern-.75\wd1\box1}
{\setbox1=\hbox{\sf Z}\copy1\kern-.75\wd1\box1}
{\setbox1=\hbox{\scriptsize\sf Z}\copy1\kern-.75\wd1\box1}
{\setbox1=\hbox{\scriptsize\sf Z}\copy1\kern-.75\wd1\box1}}}
\newcommand{\complex}{ \hbox{\rm C\kern-0.45em\rule[.07em]{.02em}{.58em}%
\kern 0.43em}}

\allowdisplaybreaks

\begin{document}
%

%
%
%
\title{Multimodal Online Federated Learning with Modality Missing in Internet of Things}

\author{{Heqiang Wang, Xiang Liu, Xiaoxiong Zhong, Lixing Chen, Fangming Liu, Weizhe Zhang} 
\thanks{
H. Wang, X. Liu, X. Zhong, F. Liu and W. Zhang are with Peng Cheng Laboratory, Shenzhen, 518066, China.
L. Chen is with Shanghai Jiao Tong University, Shanghai, 200240, China.} 
\thanks{Corresponding Authors: Xiang Liu, Xiaoxiong Zhong.}
} 

\maketitle

\begin{abstract}
The Internet of Things (IoT) ecosystem generates vast amounts of multimodal data from heterogeneous sources such as sensors, cameras, and microphones. As edge intelligence continues to evolve, IoT devices have progressed from simple data collection units to nodes capable of executing complex computational tasks. This evolution necessitates the adoption of distributed learning strategies to effectively handle multimodal data in an IoT environment. Furthermore, the real-time nature of data collection and limited local storage on edge devices in IoT call for an online learning paradigm. To address these challenges, we introduce the concept of Multimodal Online Federated Learning (MMO-FL), a novel framework designed for dynamic and decentralized multimodal learning in IoT environments. Building on this framework, we further account for the inherent instability of edge devices, which frequently results in missing modalities during the learning process. We conduct a comprehensive theoretical analysis under both complete and missing modality scenarios, providing insights into the performance degradation caused by missing modalities. To mitigate the impact of modality missing, we propose the Prototypical Modality Mitigation (PMM) algorithm, which leverages prototype learning to effectively compensate for missing modalities. Experimental results on two multimodal datasets further demonstrate the superior performance of PMM compared to benchmarks.
\end{abstract}


%
\IEEEpeerreviewmaketitle

\section{Introduction}
The rapid expansion of the Internet of Things (IoT) \cite{atzori2010internet} has led to an unprecedented surge in data generated by a multitude of interconnected devices, including smart home appliances \cite{wang2023local}, wearable health monitors \cite{wu2020rigid}, and industry sensors \cite{wang2025denoising}. To enable intelligent services and applications across the IoT ecosystem, artificial intelligence techniques, particularly machine learning and deep learning, has become a fundamental tool for model training on large-scale IoT data. Traditionally, such training has been performed in centralized cloud platforms or data centers. However, this centralized paradigm faces significant challenges as both the scale of IoT data and the number of IoT devices continue to expand. Transferring large volumes of raw data to centralized servers imposes significant demands on network bandwidth and leads to substantial communication overhead, rendering it impractical for latency-sensitive applications such as autonomous driving \cite{levinson2011towards} and real-time healthcare monitoring \cite{al2018context}. Additionally, uploading sensitive user data to the cloud raises serious privacy concerns \cite{xiao2012security}. With the gradual evolution of IoT devices from mere data collectors to intelligent edge nodes, there is increasing potential to harness their computational capabilities to address the scalability and efficiency challenges of massive IoT deployments. In this context, federated learning (FL) \cite{lim2020federated} has emerged as a promising distributed learning paradigm. FL enables collaborative model training across devices while keeping raw data local, offering a cost-effective and privacy-preserving alternative to traditional centralized learning. By significantly reducing data transmission and ensuring local data privacy, FL presents a natural and scalable solution for deploying intelligent applications in IoT environments.

Conventional FL frameworks for IoT have primarily been designed for unimodal data. However, in practice, IoT environments are inherently multimodal, involving a diverse range of sensors that generate data across multiple modalities \cite{huang2019multimodal}, such as images from cameras, audio from microphones, and structured text or signals from various sensors. This rich multimodal data provides a more comprehensive and informative representation. To address this reality, multimodal federated learning (MFL) \cite{che2023multimodal} has emerged as a natural extension of FL, aiming to enable collaborative learning across distributed multimodal data sources. A typical approach in MFL involves deploying modality-specific encoder networks, each tailored to a particular data type, to extract meaningful feature representations from high-dimensional raw inputs. These features are then fused and passed through a head encoder, usually composed of deep neural network (DNN) layers followed by a softmax classifier, to produce the final prediction.

\begin{figure}[htp]
\vspace{-5pt}
\centering
\subfloat{\includegraphics[width=1\linewidth]{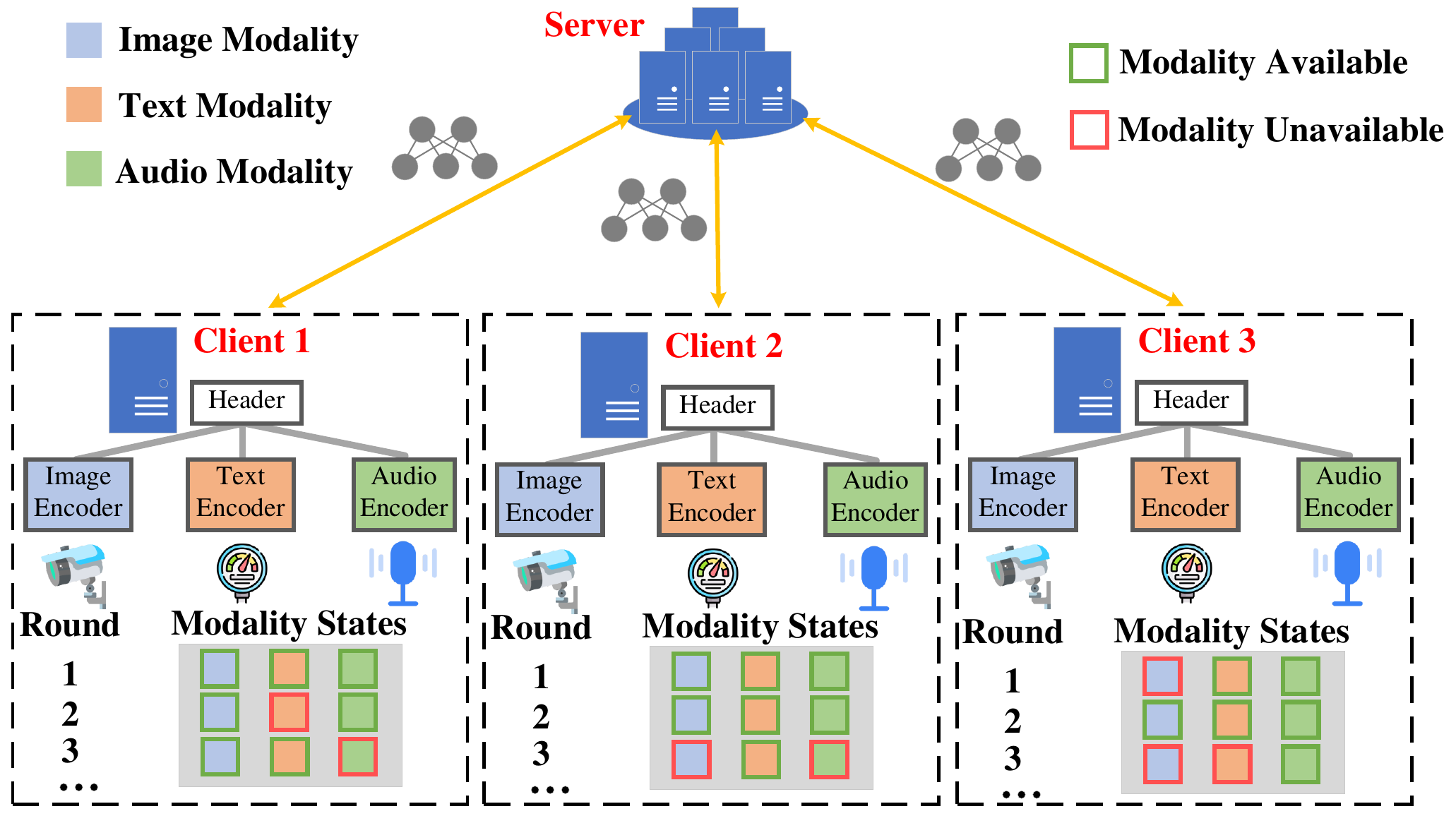}} 
\caption{IoT-Based MMO-FL with Modality Missing}  
\label{mmo_fl}
\vspace{-5pt}
\end{figure}

While traditional MFL has primarily focused on offline settings with fixed datasets, real-world IoT applications are inherently dynamic. IoT devices operate continuously, generating streaming data in real time, which makes online learning a more realistic paradigm. Furthermore, due to the instability failures of IoT sensors, certain modalities may be missing during data collection, introducing additional challenges. To address these issues, we introduce the concept of Multimodal Online Federated Learning (MMO-FL) and investigate solutions for handling modality missing in this scenario. An overview of the MMO-FL framework with missing modalities is illustrated in Fig.~\ref{mmo_fl}. Given the resource constraints of edge devices in IoT, conventional modality reconstruction techniques, such as those relying on transformers \cite{xiong2023client} or large pre-trained models \cite{che2024leveraging} are often computationally prohibitive. Therefore, it becomes critical to develop efficient modality mitigation strategies that are compatible with the limited computational capabilities of edge devices, enabling robust and scalable deployment in IoT scenarios. Our primary contributions are summarized as follows:

\begin{enumerate}
    \item We introduce the concept of MMO-FL, designed for IoT scenarios to address the challenges of FL with multimodal and streaming data. Building upon this foundation, we further investigate the problem of missing modalities, a prevalent issue in IoT environments caused by sensor instability and intermittent data acquisition failures.
    
    \item We provide an in-depth theoretical analysis of MMO-FL, investigating the dynamics of online learning and multimodal learning within the FL framework. In particular, we compare the scenarios with and without missing modalities, highlighting their differences in terms of regret bounds.
    
    \item To address the modality missing problem, we propose the Prototypical Modality Mitigation (PMM) algorithm, specifically designed to handle missing modality data in MMO-FL. Inspired by prototype learning, PMM is tailored to support dynamic data streams by enabling effective construction and substitution of modality-specific prototypes throughout the learning process.
    
    \item We evaluate the proposed PMM algorithm within the MMO-FL framework using two multimodal datasets: UCI-HAR and MVSA-Single. Experimental results demonstrate that PMM consistently outperforms baseline methods in effectively addressing the missing modality problem, achieving superior overall learning performance.
\end{enumerate}

The rest of this paper is organized as follows. Section II reviews related work on Online FL, MFL, and Prototype FL. Section III presents the system model and problem formulation. In Section IV, we describe the workflow of MMO-FL in the presence of missing modalities. Section V provides a theoretical regret analysis of the MMO-FL framework. Section VI details the proposed PMM algorithm. Experimental results are discussed in Section VII. Finally, Section VIII concludes the paper.

\section{Related Work}
\subsection{Online Federated Learning}
Online learning is designed to process data sequentially and update models incrementally, making it well-suited for applications involving continuously arriving data and the need for real-time model adaptation \cite{sahoo2017online}. These methods offer computational efficiency and eliminate the requirement of having access to the full dataset in advance, rendering them particularly suitable for memory-constrained IoT environments. In the context of FL, online federated learning (OFL) has emerged as a promising paradigm that extends online learning principles to distributed networks of decentralized learners \cite{hong2021communication}. A distinguishing feature of OFL, compared to traditional offline FL, is its emphasis on minimizing long-term cumulative regret rather than static optimization objectives during local updates. Although OFL remains relatively underexplored, several notable studies have recently advanced the field. For instance, \cite{kwon2023tighter} proposes a communication-efficient OFL algorithm that balances reduced communication overhead with strong learning performance. Similarly, \cite{mitra2021online} introduces FedOMD, an OFL method designed for uncertain environments, capable of handling streaming data without relying on assumptions about loss function distributions. While these works focus primarily on the horizontal federated learning (HFL) setting, \cite{wang2023online} explores the vertical federated learning (VFL) context, proposing an online VFL framework tailored to cooperative spectrum sensing and achieving sublinear regret. Further extending to real-world industrial applications, \cite{wang2025denoising} addresses challenges such as noise interference and device heterogeneity in online VFL systems. However, all the aforementioned approaches are limited to unimodal online federated learning. In practice, multimodal data is pervasive in IoT applications, where information from diverse types of sensors must be jointly leveraged. To bridge this gap, this work pioneers the study of MMO-FL, with the goal of enhancing the robustness and applicability of online federated learning in complex, multimodal IoT environments.

\subsection{Multimodal Federated Learning}
MFL aims to train task-relevant models on multimodal data distributed across multiple clients, thereby enabling the effective utilization of diverse data sources. With the growing interest in MFL, a variety of algorithms have been proposed to address its unique challenges and improve learning performance. One prominent challenge is modality heterogeneity, where different clients possess access to varying subsets of modalities. This inconsistency complicates model aggregation and hampers effective knowledge sharing. Several studies, such as \cite{ouyang2023harmony, chen2024feddat, chen2022fedmsplit}, have explored strategies for heterogeneous modality fusion to address this issue. Another critical challenge involves optimizing modality selection for training under constrained computational and communication resources. To tackle this, \cite{bian2024prioritizing} proposes MPriorityFed, an adaptive resource allocation framework that improves computational efficiency by prioritizing modality encoders based on their relevance and training requirements. Beyond above two challenges, a particularly pressing challenge in MFL is maintaining robust performance in the presence of missing modalities \cite{ma2021smil}. Missing data can result from incomplete data collection \cite{zhang2022m3care}, sensor failures \cite{maheshwari2024missing}, or privacy restrictions \cite{ma2021smil}, all of which degrade the effectiveness of conventional MFL frameworks. For example, \cite{wang2025clusmfl} introduces a cluster-enhanced method that utilizes feature clustering to address missing modalities in brain imaging analysis, while \cite{le2025cross} proposes the MFCPL framework, which leverages cross-modal prototypes to enhance knowledge transfer at both modality-shared and modality-specific levels. Despite these advancements, existing approaches are predominantly designed for offline learning scenarios with static datasets. They fail to address the challenges of online learning settings, where data arrives sequentially and models must adapt in real time. To bridge this gap, our work focuses on developing effective modality mitigation algorithms specifically tailored for MMO-FL, aiming to mitigate the impact of missing modalities in dynamic and distributed IoT scenarios.

\subsection{Prototype Federated Learning}
Prototype Federated Learning (PFL) has emerged as a promising solution for addressing data heterogeneity and personalization challenges in federated learning. Unlike conventional FL approaches that rely on aggregating model parameters or gradients, PFL focuses on exchanging class-level prototypes, which are representative feature embeddings of data classes, between clients and the central server. This prototype-sharing strategy enhances communication efficiency and improves model robustness, particularly under non-IID data distributions \cite{zhang2024fedgmkd, tan2022fedproto, zhou2025fedsa, tan2022federated}. A representative example is FedProto \cite{tan2022fedproto} introduces a framework in which clients compute local class prototypes and transmit them to the server. The server then aggregates these into global prototypes and redistributes them to clients, thereby aligning local updates with global objectives and mitigating the adverse effects of data heterogeneity. Extending this concept, FedGPD \cite{wu2024global} incorporates knowledge distillation, using global prototypes as distilled information to guide local training. This method eliminates the dependency on public proxy datasets and enhances generalization across clients with diverse data distributions. Furthermore, FedGMKD \cite{zhang2024fedgmkd} advances prototype learning by modeling prototype features using Gaussian Mixture Models and applying discrepancy-aware aggregation, which dynamically weights client contributions based on both data quality and volume, resulting in improved global performance under heterogeneous conditions. Beyond tackling heterogeneity, recent studies have demonstrated that prototype learning is also effective for addressing modality missing problems, especially in centralized learning scenarios \cite{jin2023rethinking, li2024correlation} and some distributed learning scenarios \cite{bao2023multimodal}. However, existing PFL research is confined to offline learning, with limited exploration of online learning, particularly in OFL contexts. OFL introduces unique challenges such as streaming data and non-stationary distributions, which require new adaptations of prototype-based strategies. Addressing the modality missing problem through prototype learning within the context of MMO-FL remains an open and important direction for future research.

\section{System Model }
Before presenting the details of the system model, we first summarize the key notations used later in Table~\ref{table1}.

\begin{table}[ht]
\vspace{-5pt}
\centering
\caption{Key Notations}
\begin{tabular}{cc}
\toprule
\textbf{Symbol} & \textbf{Semantics} \\
\midrule
$K$ & The number of clients \\
$M$ & The number of modalities \\
$T$ & The number of global rounds \\
$N$ & The number of data samples collected by client\\
$E$ & The number of local iterations \\
$C$ & The number of classes \\
$\theta^m$ & The modality encoder\\
$\theta^0$ & The head encoder\\
$\eta$ & The learning rate \\
$Z$ & The feature extractor\\
$\Theta$ & The overall model \\
$\mathcal{S}_t$ & The set of clients without modality missing\\
$\textbf{G}_k^t$ & The overall gradient \\
$\beta_k^t$ & The proportion of available modalities \\
$p_{c}^{t, m}$ & The temporal global prototype  \\
$\bar{p}_{c}^{t, m}$ & The consistent global prototype \\
$\bar{\mathcal{P}}^t$ & The consistent global prototype collection \\
$\lambda$ & The ratio of modality missing occurs\\
$\alpha$ & The Non-IID level \\
\bottomrule
\end{tabular}
\label{table1}
\vspace{-5pt}
\end{table}

Consider an IoT-based smart factory scenario consisting of a cloud server and $K$ workstations, each acting as a client. Each workstation is equipped with a diverse set of sensors to monitor factory conditions across multiple modalities at different locations, including vision sensors, acoustic sensors, and temperature and humidity sensors, as illustrated in Fig.~\ref{mmo_fl}. Additionally, given the real-time data collection from these sensors, the objective is to enable cooperative training across the clients on a unified global model using multimodal streaming data. This setting forms the foundational structure of the MMO-FL problem. During the training process, the sensors at each client continuously collect new multimodal data over time, with the timeline divided into discrete periods, denoted as $t = 1, 2, ..., T$. For simplicity, each time period is also treated as a global round.

In each global round $t$, each client $k \in \mathcal{K}$ collects the current local training dataset consisting of $N$ data samples from $M$ modalities with no missing modalities. The dataset is denoted as $ \mathcal{D}_k^t = \left ( X_k^{t, 1}, X_k^{t, 2}, \dots, X_k^{t, M}; Y_k^t \right ) = \left\{ \left ( x_{k, n}^{t, 1}, x_{k, n}^{t, 2}, \dots, x_{k, n}^{t, M}; y_{k, n}^t \right )\right\}_{n=1}^{\left | \mathcal{D}_k^t\right |}$. Here, $x_{k, n}^{t, m}$ represents the $m^{th}$ modality data of the $n^{th}$ sample in client $k$ collected at global round $t$, and $y_{k, n}^t$ denotes the corresponding label. We also define $\mathcal{D}^t = \sum_{k=1}^K \mathcal{D}_k^t$ as overall dataset aggregated across all clients at global round $t$. Without loss of generality, we assume that each client $k$ collects exactly $N$ training samples per global round $t$.

In the MMO-FL, the overall model that needs to be trained by each client and aggregated at the server is divided into two components: the modality encoders $\theta^1, \dots, \theta^M$, which extract the feature-level information from the raw data, and the head encoder $\theta^0$, which integrates these features to generate the final prediction. Each modality encoder $\theta^m$ may adopt a distinct architecture tailored to the characteristics of the corresponding modality. The feature vector for the $m^{th}$ modality of client $k$ at round $t$, extracted by modality encoder $\theta^m$, is denoted as $ Z_k^{t, m} = \theta^m(X_k^{t, m} )$. The header encoder $\theta^0$ then synthesizes the outputs $Z_k^{t, m}$ from each modality encoder to fulfill the final learning objective.  Based on these definitions, the loss function for the collective training dataset at round $t$ is expressed as follows: 
\begin{align}
F_t(\Theta, \mathcal{D}^t)  = \frac{1}{K}\sum_{k=1}^K f_t\left (\theta^0 \left (Z_k^{t, 1}, \dots, Z_k^{t, M}  \right ), Y_k^t \right)
\end{align}
where $\Theta = \left\{\theta^0, \theta^1, \dots, \theta^M \right\}$ represents the overall model. The term $\theta^0 \left (Z_k^{t, 1}, \dots, Z_k^{t, M}  \right )$ denotes the predicted labels through head encoder, and $f_t$ is the loss function that measures the discrepancy between the predicted labels and the actual labels. It is important to note that the above loss function corresponds to a single global round.

Since the training process is based on dynamically collected real-time data rather than a static dataset, adopting an online learning paradigm is essential. Let the overall model at each global round be represented as $\Theta^1, \ldots, \Theta^T$. The learning regret, $\text{Reg}_T$ is defined to quantify the gap between the cumulative loss incurred by the learner and the cumulative loss of an optimal fixed model in hindsight. Specifically:
\begin{align}
  \text{Reg}_T = \sum_{t=1}^{T}  F_t (\Theta^t; \mathcal{D}^t ) - \sum_{t=1}^{T} F_t (\Theta^*; \mathcal{D}^t) \label{regret}
\end{align}
Where $\Theta^* = \arg\min_\Theta \sum_{t=1}^{T} F_t (\Theta; \mathcal{D}^t)$ represents the optimal fixed model selected in hindsight. Our objective is to minimize the learning regret, which equates to minimizing the cumulative loss. Importantly, if the learning regret grows sublinearly with respect to $T$, it indicates that the online learning algorithm can progressively reduce the training loss asymptotically.

The fixed optimal strategy in hindsight refers to an idealized solution determined by a centralized entity with complete prior knowledge of all per-round loss functions. In our problem, achieving such an optimal strategy would require access to future information, including the entirety of the streaming datasets collected in subsequent rounds. However, due to the stochastic and unpredictable nature of real-time data collection, this information is inherently unavailable in practice. Consequently, the full loss functions are not known in advance and evolve dynamically over time. Therefore, regret just serves as a metric to quantify the performance gap between the proposed algorithm and the theoretical optimal strategy in our theoretical analysis. For experimental validation, we will evaluate the effectiveness of our algorithm using practical metrics such as test accuracy.

\section{The Workflow of MMO-FL with Missing Modalities }
In this section, we present the overall workflow of the MMO-FL framework and analyze the challenges introduced by missing modalities during the learning process. This discussion serves as the foundation for the subsequent theoretical analysis and the development of algorithms aimed at mitigating the impact of modality missing. We place particular emphasis on highlighting the distinctions between MMO-FL and conventional unimodal FL approaches.

MMO-FL builds upon the HFL framework, where each client collects independent data samples and the server is responsible for aggregating the global model. In conventional HFL, clients share the same feature space but have different sample spaces. However, in the case of missing modalities, the feature space across clients may not remain consistent during the whole learning process due to the presence of modality missing, differing from the standard HFL framework. During each global round $t \in \mathcal{T}$, where $\mathcal{T} = \{ 0, 1, 2, \dots, T-1 \}$, all clients execute a fixed number of local training iterations, represented by the parameter $E$. The index $\tau = 0, 1, 2, ..., E$ is used to track these local iterations. Then each global round $t$ consists of a sequence of coordinated steps carried out across clients and the server. A visual illustration of the timeline for a single global round is provided in Fig.~\ref{timeline}.

\begin{figure}[htp]
\vspace{-5pt}
\centering
\subfloat{\includegraphics[width=1\linewidth]{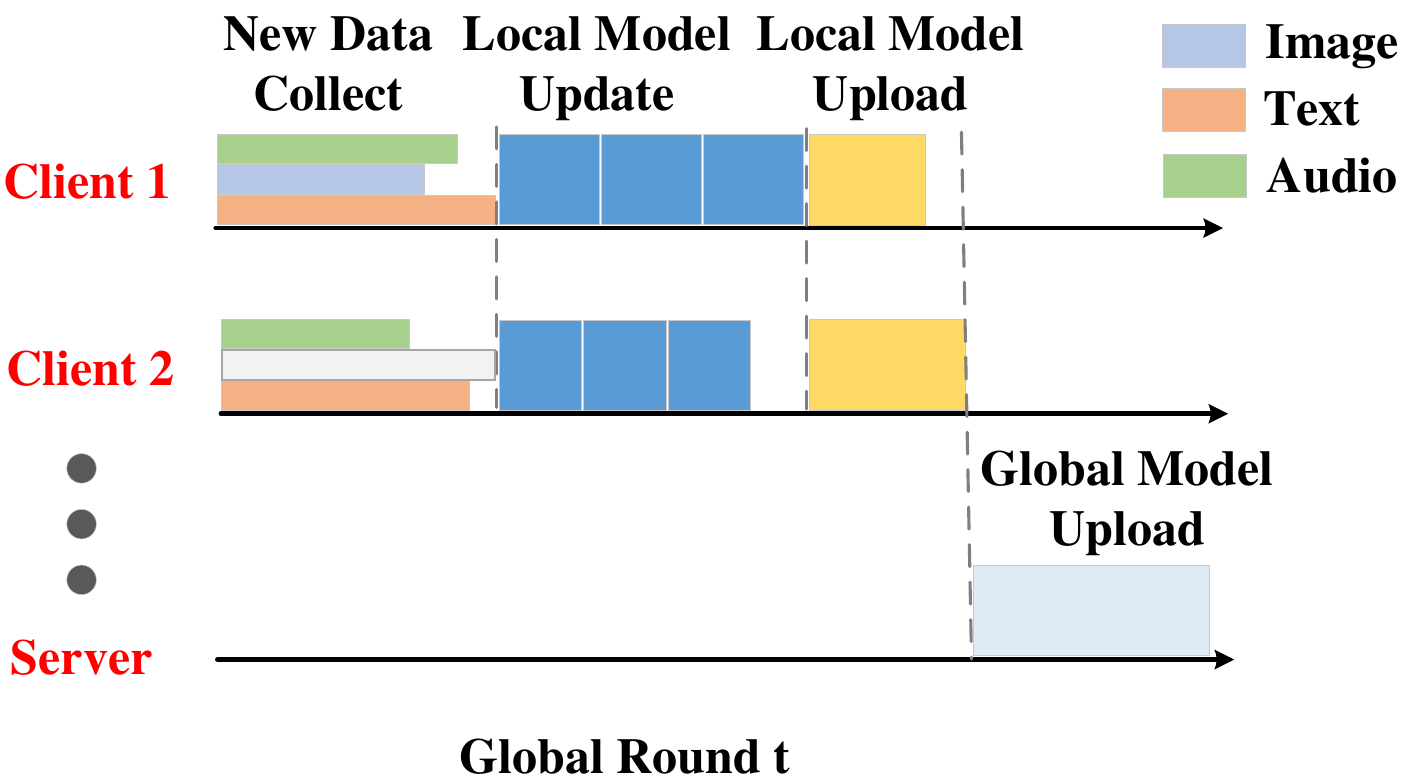}} 
\caption{The time diagram of MMO-FL in one global round }  
\label{timeline}
\vspace{-5pt}
\end{figure}

\subsubsection{\textbf{Client - New Data Collect}} At the beginning of each global round $t$, each client $k$ collects new local training data $ \mathcal{D}_k^t$, utilizing the available set of modalities $\mathcal{M}_k^t$, determined by the operational status of the sensors on the client. Due to the inherent instability of edge devices in IoT environments, some sensors on clients may fail to collect new data during certain global rounds, leading to the appearance of modality missing. For convenience, we define $\bar{\mathcal{M}}_k^t = \mathcal{M} \backslash \mathcal{M}_k^t$ as the set of missing modalities for client $k$ at global round $t$. Unlike the traditional modality missing issue in offline learning settings, where the absence of a specific modality remains consistent throughout the training process, the online setting introduces variability, with the missing modalities for a client potentially changing between global rounds. This dynamic nature adds an additional layer of complexity to the learning process and renders conventional solutions designed for static offline scenarios ineffective.

\subsubsection{\textbf{Client - Local Model Update}} Each client $k$ uses the current global model $\Theta^{t}$, provided by the server, as the initial model to train a new local model based on the current collected training dataset $\mathcal{D}_k^t$. Each client performs $E$ iterations of online gradient descent (OGD) using the full training dataset. The model update rule varies depending on whether the client has access to all modalities or is experiencing modality missing. For clarity, we define the set $\mathcal{S}_t$ as the set of clients with complete modality availability in global round $t$, while the remaining clients, denoted by $\mathcal{K}\backslash \mathcal{S}_t$, represent those with missing modalities.

\textbf{For Full Modality Client:} In this case, the data collected by client $k \in \mathcal{S}_t$ includes all modalities $\mathcal{M}$, allowing the client to utilize the complete local dataset for OGD updates. Each modality data is processed through its respective modality encoder, and the final result is produced by the head encoder. In each global round, the client performs $E$ local training iterations, following the update process described by the following equations:
 \begin{align}
        &\Theta_k^{t,0} = \Theta^{t} \notag\\
        &\Theta_k^{t, \tau + 1} = \Theta_k^{t, \tau } - \eta \textbf{G}_{k}^{t, \tau }, \quad \forall \tau = 1, ..., E \notag\\
        &\Theta_k^{t+1} = \Theta_k^{t, E} 
\end{align}
Here, $\textbf{G}_{k}^{t, \tau } = \nabla F_{t}(\Theta_k^{t, \tau }, \mathcal{D}^t_k)$  represents the gradient computed on the current local dataset containing all modalities, and $\eta$ denotes the corresponding learning rate. 

\textbf{For Missing Modality Client:} In this case, the data collected by client $k \in \mathcal{K}\backslash \mathcal{S}_t$ contains missing modalities with modalities set $\bar{\mathcal{M}}_k^t$ caused by the failure of data collection from the corresponding type of sensors. Consequently, the client can only perform local updates using the partial modality dataset $\tilde{\mathcal{D}}^t_k =  {\mathcal{D}}^t_k \backslash  \left\{ X_k^{t, m'} \right\}_{m' \in \bar{\mathcal{M}}_k^t}$. The available modalities are processed through their corresponding modality encoders, and the final output is generated by the head encoder. The client also executes $E$ local training iteration, following the update procedure detailed in the equations below:
 \begin{align}
        &\Theta_k^{t,0} = \Theta^{t} \notag\\
        &\tilde{\Theta}_k^{t, \tau + 1} = \tilde{\Theta}_k^{t, \tau } - \eta \tilde{\textbf{G}}_{k}^{t, \tau }, \quad \forall \tau = 1, ..., E \notag\\
        &\tilde{\Theta}_k^{t+1} = \tilde{\Theta}_k^{t, E} 
\end{align}
Here, we denote the gradient computed on the current local dataset with missing modalities as:
\begin{align}
\tilde{\textbf{G}}_{k}^{t, \tau } = \left [  {\left ( \tilde{\textbf{G}}_{k}^{t, \tau, 0} \right )}^{\top}, \dots  {\left ( \tilde{\textbf{G}}_{k}^{t, \tau, m}  \right )}^{\top},  \dots  {\left ( \tilde{\textbf{G}}_{k}^{t, \tau, M}  \right )}^{\top} \right ]^{\top}
 \end{align}
where
\begin{align}
\tilde{\textbf{G}}_{k}^{t, \tau, m } = \left\{\begin{matrix}
\nabla_m F_{t}(\Theta_k^{t, \tau }, \mathcal{D}^t_k) &  \textit{if} \quad m \in \mathcal{M}_k^t \\
0 & otherwise \\
\end{matrix}\right.
\end{align}

To highlight the impact of missing modality data on the model, $\tilde{\Theta}^{t+1}_k $ is used to denote local models with the impact of missing modalities. In this case, the parameters corresponding to the absent modalities remain unchanged, as they cannot be updated due to the failure in data collection and the absence of corresponding gradient updates.

\subsubsection{\textbf{Client - Local Model Upload}} Each client will upload the corresponding local model to the server after finishing the $E$ iterations of local model update. For clients $k \in \mathcal{S}_t$ with full modalities, the client will upload the local model $\Theta^{t+1}_k$ trained using data from all modalities. For clients $k \in \mathcal{K}\backslash \mathcal{S}_t$ with missing modalities, the client will upload the local model $\tilde{\Theta}^{t+1}_k$, which has been trained using available data with partial modalities. In the subsequent discussion, we also slightly abuse the notation $\tilde{\Theta}^{t+1}_k$ to also represent the local model obtained after applying a mitigation strategy for missing modalities. It is important to note that this model is inherently different from the one trained using complete modality data.

\subsubsection{\textbf{Server - Global Model Update}} The server updates the global model by using the local model updates from the clients, as given by the following equation:
\begin{align}
    \Theta^{t+1} = \frac{1}{K}\left(\sum_{k \in \mathcal{S}_t} \Theta^{t+1}_k + \sum_{k \in \mathcal{K} \backslash \mathcal{S}_t} \tilde{\Theta}^{t+1}_k\right)
\end{align}
The server then broadcasts the updated global model to all clients for the next global round. This process continues until the pre-set total number of global rounds is reached.

From the above MMO-FL workflow, it is evident that a central challenge introduced by modality missing lies in designing an effective mitigation strategy. The objective is to ensure that the model updated using reconstructed modalities closely approximates the one trained with complete modality data, i.e., $\tilde{\Theta}^{t+1}_k \rightarrow \Theta^{t+1}_k$. In the following, we first provide a detailed theoretical analysis to quantify the impact of modality missing on the regret bound of MMO-FL. Building on these insights, we then introduce the Prototypical Modality Mitigation algorithm, which is designed to effectively alleviate the adverse effects of missing modalities.

\section{Theoretical Analysis }
In this section, we provide a comprehensive regret analysis of the proposed MMO-FL algorithm. We will conduct the regret analysis in three steps. First, we will examine the regret bound for the case where the local iteration is $ E= 1$  without the impact of modality missing. Next, we will extend the analysis to the scenario where the local iteration is $ E > 1$. Finally, we will explore the regret bound when $ E > 1$ while accounting for the impact of modality missing.

\subsection{MMO-FL with local iteration $E=1$ and without modality missing}
To facilitate our analysis, we first introduce several additional definitions. After applying some basic transformations to the global model update equation above, we derive an alternative form of the global model update equation for the case where the local iteration is $E=1$ and no modality is missing, as follows:
\begin{align}
\Theta^{t+1} = \Theta^{t} - \frac{\eta }{K} \sum_{k=1}^K \textbf{G}^{t}_k
\end{align}
where $\textbf{G}^{t}_k$ represents the gradient of the local overall model for client $k$ across all $M$ modalities, given by:
\begin{align}
\textbf{G}^{t}_k &= \left [  {\left ( \textbf{G}^{t, 0}_k \right )}^{\top}, {\left ( \textbf{G}^{t, 1}_k \right )}^{\top}, \dots  {\left ( \textbf{G}^{t, m}_k  \right )}^{\top},  \dots  {\left ( \textbf{G}^{t, M}_k  \right )}^{\top} \right ]^{\top}
\end{align}

Subsequently, we will introduce the assumptions that are standard for analyzing online convex optimization, as referenced in \cite{park2022fedqogd}. Several of these assumptions are defined at the modality level to align with the specific requirements of our multimodal parameter formulation.

\begin{assumption}
For any $\mathcal{D}^t$, the loss function $F_t (\Theta; \mathcal{D}^t)$ is convex with respect to $\Theta$ and differentiable. 
\label{assm:cov}
\end{assumption}

\begin{assumption}
The loss function is $L$-Lipschitz continuous, the partial derivatives for each modality satisfies: $\left \| \nabla_{m}  F_t (\Theta)   \right \|^2 \leq L^2$. 
\label{assm:bpd}
\end{assumption}

Assumption \ref{assm:cov} ensures the convexity of the function, enabling us to leverage the properties of convex optimization. Assumption \ref{assm:bpd} constrains the magnitude of the loss function's partial derivatives at the modality level. Based on these assumptions, we can derive the following Theorem \ref{thm:general}.

\begin{theorem}\label{thm:general}
Under Assumption 1-2, MMO-FL with local iterations $E=1$ and excluding the impact of modality missing, achieves the following regret bound:
\begin{align}
 & {Reg}_T   = \sum_{t=1}^{T} \sum_{k=1}^K \mathbb{E}_t \left [ F_t (\Theta^{t}; \mathcal{D}^t_k)  \right ] - \sum_{t=1}^{T} \sum_{k=1}^K F_t (\Theta^*; \mathcal{D}^t_k) \notag \\
 & \leq \frac{K  \left \| \Theta^{1} - \Theta^* \right \|^2}{2 \eta }  + \frac{\eta T K (M+1) L^2}{2}
\end{align}
\end{theorem}
\begin{proof}
The proof can be found in Appendix A.
\end{proof}

According to Theorem \ref{thm:general}, by setting $\eta = \mathcal{O}(1/\sqrt{T})$, the MMO-FL can achieve a sublinear regret bound over $T$ time rounds, specifically $\mathcal{O}(\sqrt{T})$. A sublinear regret bound indicates that the average regret per round, defined as the regret divided by the number of rounds, approaches zero as the number of rounds increases indefinitely. This suggests that the algorithm progressively refines its performance by learning from its errors. 

\subsection{MMO-FL with local iterations $E>1$ and without modality missing}
In this section, we extend the proof to the case where the local iteration $E>1$ and no modality is missing. Following a similar transformations approach, we derive the global model update equation as follows:
\begin{align}
\Theta^{t+ 1, 0} = \Theta^{t, 0} - \frac{\eta}{K} \sum_{k=1}^K \sum_{\tau = 0}^{E -1} {\textbf{G}}^{t, \tau}_k
\end{align}
where $\textbf{G}^{t, \tau}_k$ denotes the gradient of the local overall model for client $k$ across all $M$ modalities for round $t$ and local iteration $\tau$. Due to the presence of multiple local iterations $E > 1$, we introduce two additional assumptions to support the regret analysis of this case.

\begin{assumption}
The partial derivatives, corresponding to the consistent loss function, fulfills the following condition:
\begin{align}
\left \|  {\textbf{G}}_{k}^{t, \tau'} - {\textbf{G}}^{t, \tau}_{k} \right \| \leq \varphi \left \| \Theta_k^{t, \tau'} - \Theta_k^{t, \tau} \right \| \notag
\end{align}
\label{assm:gradient-change}
where $\tau'$ and $\tau$ indicate that they correspond to different local iterations.
\end{assumption}

For the purposes of subsequent theoretical analysis, we consider a $D$-dimensional vector for each modality in both the overall gradient and model. We define an arbitrary vector element $d \in [1, D]$ in overall gradient for modality $m$ as ${\textbf{G}}_{k, d}^m$, and similarly, the arbitrary vector element $d \in [1, D]$ in overall model for modality $m$ is denoted as $\Theta_{k, d}^m$. Including the head encoder $m = 0$, each overall gradient and model consists of a total of $(M + 1)D$ vector elements.

\begin{assumption}
The arbitrary vector element $d$ in the overall model $\Theta_{k, d}^m$ for any modality $m $ is bounded by: $\left | \Theta_{k, d}^m \right | \leq \sigma  $.
\label{assm:model-variant}
\end{assumption}

Assumption \ref{assm:gradient-change} ensures that the variation in the partial derivatives is confined within a specific range, which aligns with the model variation over two different local iterations that maintain a consistent loss function. This approach effectively utilizes the concept of smoothness. Lastly, Assumption \ref{assm:model-variant} specifies the permissible range for any vector element in the overall model. Then we can obtain the following Theorem \ref{thm:general2}.

\begin{theorem}\label{thm:general2}
Under Assumption 1-4, MMO-FL with local iterations $E>1$ and excluding the impact of modality missing, achieves the following regret bound:
\begin{align}
 & {Reg}_T   = \sum_{t=1}^{T} \sum_{k=1}^K \mathbb{E}_t \left [ F_t (\Theta^{t,0}; \mathcal{D}^t_k)  \right ] - \sum_{t=1}^{T} \sum_{k=1}^K F_t (\Theta^*; \mathcal{D}^t_k) \notag \\
 & \leq \frac{K  \left \| \Theta^{1, 0} - \Theta^* \right \|^2}{2 \eta E}  + \frac{\eta T K E (M+1) L^2}{2} \notag  \\
 & + 2\eta T D E K (M+1)^2 \varphi \sigma L 
\end{align}
\end{theorem}
\begin{proof}
The proof can be found in Appendix C.
\end{proof}
According to Theorem \ref{thm:general2}, by setting $\eta = \mathcal{O}(1/\sqrt{T})$, the MMO-FL can also achieve a sublinear regret rate $\mathcal{O}(\sqrt{T})$. However, in contrast to Theorem \ref{thm:general}, the regret bound includes an additional term resulting from the presence of multiple local iterations. In the following analysis, we will consider the effect of modality missing.
\subsection{MMO-FL with local iterations $E>1$ and modality missing}
In this section, we extend the proof to the more complex and common scenario where the local iteration $E>1$ and modality missing is present. We derive the global model update equation as follows:
\begin{align}
\Theta^{t+ 1, 0} = \Theta^{t, 0} - \frac{\eta}{K} \sum_{k=1}^K \sum_{\tau = 0}^{E -1} \tilde{\textbf{G}}^{t, \tau}_k
\end{align}
where $\tilde{\textbf{G}}^{t, \tau}_k$ represents the gradient of the local overall model for client $k$ with partial modalities $\mathcal{M}_k^t$ at round $t$ and local iteration $\tau$. Obviously, $\mathcal{M} \neq \mathcal{M}_k^t$ holds for most of the time.

Since each client may have a different modality state in each global round, we introduce an additional assumption to constrain the maximum upper bound of modality missing. suppose we define $\beta_k^t = \frac{| \mathcal{M}_k^t|}{M+1}$ as the proportion of available modalities for client $k$ at round $t$ relative to the total number of modalities.

\begin{assumption}
For any client $k$ at any global round $t$, the minimum number of available modalities is bounded as follows: $\beta_k^t \geq \beta$, where $\beta$ is a constant between $[0, 1]$.
\label{assm:model-missing}
\end{assumption}

Assumption \ref{assm:model-missing} guarantees the most extreme case of modality missing, ensuring that no client will experience a complete lack of modal data at any global round, which would otherwise prevent training. With the above assumptions in place, we can derive the following Theorem \ref{thm:general3}.

\begin{theorem}\label{thm:general3}
Under Assumption 1-5, MMO-FL with local iterations $E>1$ and including the impact of modality missing, achieves the following regret bound:
\begin{align}
 & {Reg}_T   = \sum_{t=1}^{T} \sum_{k=1}^K \mathbb{E}_t \left [ F_t (\Theta^{t,0}; \mathcal{D}^t_k)  \right ] - \sum_{t=1}^{T} \sum_{k=1}^K F_t (\Theta^*; \mathcal{D}^t_k) \notag \\
 & \leq \frac{K  \left \| \Theta^{1, 0} - \Theta^* \right \|^2}{2 \eta E} + (5 -2 \beta) \eta T K  (M+1) L^2  \notag  \\
 & + 2\eta T D K  (M+1)^2 \varphi \sigma L + \frac{2 (1 - \beta) (M+1) T D K  \sigma L }{E}
\end{align}
\end{theorem}
\begin{proof}
The proof can be found in Appendix E.
\end{proof}

Based on Theorem 3, by choosing the learning rate as $\eta = \mathcal{O}(1/\sqrt{T})$, MMO-FL with missing modality can achieve a regret bound of $\mathcal{O}(\sqrt{T} + T (1 - \beta))$ over $T$ time rounds. Our analysis reveals that the term $\mathcal{O}(T (1 - \beta))$, which accounts for the missing modality, plays a crucial role in determining whether a sublinear regret bound can be achieved. Notably, if we set $\beta = 1$, meaning no modality is missing for any client in any round, MMO-FL clearly achieves a regret bound of $\mathcal{O}(\sqrt{T})$. However, in practical IoT environments where device stability cannot always be guaranteed, this term cannot be entirely eliminated. This occurs because missing data for a specific modality prevents the corresponding partial gradient from being updated. Consequently, the most effective way to mitigate the impact of missing modalities is to approximate the absent modal information. In the following section, we introduce the Prototypical Modality Mitigation algorithm, which is specifically designed to address this problem.

\section{Prototypical Modality Mitigation Algorithm}
In this section, we present the details of the Prototypical Modality Mitigation algorithm (PMM). As discussed in the previous section, the regret bound of MMO-FL with missing modalities is influenced by the modality missing ratio. However, since the occurrence of missing modalities is often unpredictable and uncontrollable in real-world IoT settings, it becomes essential to mitigate their impact. A natural approach is to reconstruct or approximate the missing modalities during training process. Drawing inspiration from prototype learning, we identify prototypes as an effective solution to address this challenge, as they encapsulate the underlying semantics of each class and can serve as reliable substitutes for absent modality data.


\begin{figure*}[htbp]
    \centering
    \includegraphics[width=0.8\textwidth]{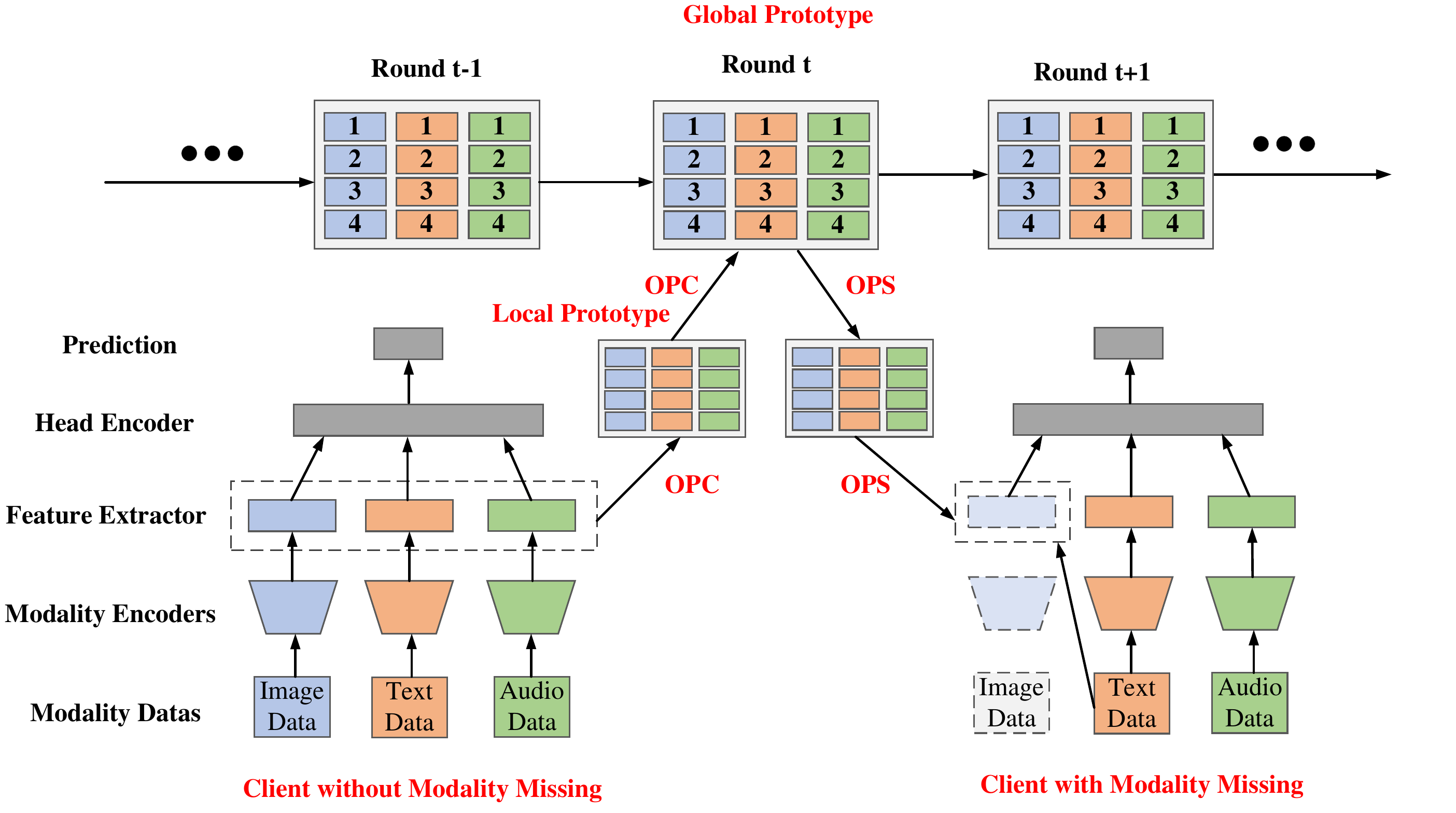} 
    \caption{Illustration of PMM: Online Prototypes Construction (OPC) involves generating prototypes from continuously evolving data throughout the MMO-FL learning process, ensuring a dynamic representation of each modality. Online Prototypes Substitution (OPS), on the other hand, leverages these prototypes to compensate for missing modalities in real-time, enabling effective modality reconstruction for clients encountering incomplete modality data. }
    \vspace{-10pt}
    \label{pmm}
\end{figure*}

The proposed PMM algorithm is detailed in Algorithm~\ref{alg: pmm} and comprises two key components: Online Prototypes Construction (OPC), which builds prototypes based on dynamic data during the MMO-FL learning process, and Online Prototypes Substitution (OPS), which utilizes these prototypes for real-time modality mitigation in clients experiencing missing modalities. The detailed flowchart of the PPM algorithm is illustrated in Fig.~\ref{pmm}. The subsequent sections provide a comprehensive explanation of each component.

\begin{algorithm}
    \caption{Prototypical Modality Mitigation Algorithm} \label{alg: pmm} 
    \begin{algorithmic}[1]
        \For {$t = 1, 2, ..., T - 1$}
            \State \textbf{Client Side:}
            \For {$k = 1, 2, ..., K$}
                \If { Client with full modalities ($k \in \mathcal{S}_t$)}
                    \State Calculates the local prototype $p_{k, c}^{t, m}$  via Eq.~\ref{local_pro}.
                    \State Integrates and uploads the $p_{k, c}^{t, m}$ to server.
                \Else
                    \State Downloads and utilizes the collection $\bar{\mathcal{P}}^t$.
                    \State Builds the representations $\tilde{Z}_k^{t, m}$ via Eq.~\ref{vir_fe}.
                    \State Utilizes the $\tilde{Z}_k^{t, m}$ for class prediction.
                \EndIf
            \EndFor
            \State \textbf{Server Side:}
            \State Collects the local prototype $p_{k, c}^{t, m}$  from clients.
            \State Calculates the $p_{c}^{t, m}$ via Eq.~\ref{tgp}.
            \State Updates the $\bar{p}_{c}^{t, m}$ via Eq.~\ref{pgp}.
            \State Broadcasts the $\bar{\mathcal{P}}^t$ to all clients.
        \EndFor
    \end{algorithmic}
\end{algorithm}

\subsection{Online Prototypes Construction}
The OPC process runs concurrently within each global round. We define $c \in  \left\{1, \dots, C \right\}$ as the label classes. Accordingly, $X_{k, c}^{t, m}$ represents the set of training sample of class $c$ and modality $m$, collected by client $k$ at global round $t$. The local prototype is then defined as the average value of the features extracted by the modality encoder:
\begin{align}
p_{k, c}^{t, m} = \frac{1}{ | X_{k, c}^{t, m} |} \sum_{n \in \mathcal{N}( y_{k, n}^t = c)}  \theta^m(x_{k, n}^{t, m} ) \label{local_pro}
\end{align}
To ensure fairness among prototypes for different modalities, we introduce two additional rules: (1) Local prototypes across different modalities should follow a uniform structure, ensuring that the extracted features $\theta^m(x_{k, n}^{t, m} )$ from modality encoders remain consistent, despite differences in the original data structures of each modality. By processing data through these encoders, the extracted features from various modalities can be structurally aligned. (2) Local prototypes across modalities should be normalized to maintain a consistent magnitude, ensuring fair comparability and stability throughout the learning process.

After each training round, each client generates local prototypes for each modality and transmits them to the server. After receiving the local prototypes for each modality, the server creates a temporal global prototype for modality $m$ and class $c$ at global round $t$ as follows:
\begin{align}
p_{c}^{t, m} = \frac{1}{|\mathcal{S}_t|} \sum_{k \in \mathcal{S}_t}  p_{k, c}^{t, m} \label{tgp}
\end{align}
Given the nature of online learning, data bias may arise, and the collected data may not encompass all classes in every global round. To address this issue, it is essential to maintain a persistent global prototype that accurately captures the semantic representation of each class for every modality across rounds. We define the persistent global prototype for modality $m$ and class $c$ at global round $t$ as follows:
\begin{align}
\bar{p}_{c}^{t, m} = \frac{(t-1)\bar{p}_{c}^{t-1, m} + {p}_{c}^{t, m}}{t} \label{pgp} 
\end{align}
The persistent global prototype $\bar{p}_{c}^{t, m}$ is continuously updated and stored on the server, ensuring its availability for supporting the OPS process.

\subsection{Online Prototypes Substitution}
In this phase, the client performs prototype substitution to address missing modalities problem. After updating the persistent global prototype, the server organizes them by modality to create the persistent global prototype collection, denoted as $\bar{\mathcal{P}}^t$, defined as follows: 
\begin{align}
\bar{\mathcal{P}}^t =  
\begin{bmatrix}
\bar{p}_{1}^{t, 1} & \dots & \bar{p}_{c}^{t, 1} & \dots & \bar{p}_{C}^{t, 1} \\
\dots & \dots & \dots & \dots & \dots  \\
\bar{p}_{1}^{t, m} & \dots & \bar{p}_{c}^{t, m} & \dots & \bar{p}_{C}^{t, m} \\
\dots & \dots & \dots & \dots & \dots  \\
\bar{p}_{1}^{t, M} & \dots & \bar{p}_{c}^{t, M} & \dots & \bar{p}_{C}^{t, M} \\
\end{bmatrix}
\end{align}
Next, the server distributes this global prototype collection to all clients. When a client detects that a certain modality is missing due to uncollected data, it generates the corresponding virtual feature representations for the missing modality. This reconstruction is performed by leveraging the received global prototypes in conjunction with the class distribution inferred from the available modality data, as formulated in the following equation:
\begin{align}
\tilde{Z}_k^{t, m} = \left [ \bar{p}_{ c(k, 1)}^{t, m}, \dots \bar{p}_{ c(k, n)}^{t, m}, \dots \bar{p}_{ c(k, N)}^{t, m}\right ]   \label{vir_fe}
\end{align}
Here, $\bar{p}_{c(k, n)}^{t, m}$ denotes the persistent global prototype of modality $m$, corresponding to the class label of the $n$-th data sample from client $k$ at global round $t$. When modality $m$ is absent in the current global round, the corresponding feature representations are synthesized using the persistent global prototypes. Subsequently, the predicted labels are obtained through the head encoder as $\theta^0 \left (Z_k^{t, 1}, \dots, \tilde{Z}_k^{t, m}, \dots Z_k^{t, M}  \right )$.
This is then followed by the standard execution of the MMO-FL process, as introduced in the previous section. To account for potential overhead and additional considerations introduced by the implementation of the PMM algorithm, we present the following three additional remarks to further refine our proposed algorithm. 

\textbf{Remark 1 (Prototype Approximation):} From the perspective of the law of large numbers, a feature extractor based on persistent global prototypes can effectively represent the feature extractor for missing modalities. As the number of prototypes increases, the representation of each modality becomes more accurate. The law of large numbers suggests that as the sample size grows, the sample mean approaches the true mean of the population. Similarly, in feature extraction, using a larger number of prototypes captures the diverse characteristics of the data, reducing potential bias or distortion that could occur with a limited set of prototypes. By integrating more prototypes, the system can better approximate the feature distribution of the missing modality, compensating for missing information by leveraging the overall trends and patterns derived from the larger dataset.

\textbf{Remark 2 (Quantized Upload):} Since the persistent global prototype in the OPS phase serves as an approximate compensation mechanism, the precision of this substitution may not be highly sensitive and precise. Therefore, to further reduce communication overhead, quantization technique can be applied when transmitting the persistent global prototype collection. In our experiments, we will also investigate the impact of this strategy on overall learning performance.

\textbf{Remark 3 (Delayed Update):} Given the robustness of online learning and the approximate nature of prototype substitution, the OPC process does not need to be executed in every global round. This design choice helps reduce both computational cost and communication overhead. Accordingly, the persistent global prototype collection used in the OPS phase can be a previously stored version rather than one updated in the current round. In our experiments, we further investigate the impact of this strategy on overall learning performance.

\section{Experiment }
In this section, we will experimentally evaluate the performance of the MMO-FL algorithm. The experiments were conducted on an Ubuntu 18.04 machine equipped with an Intel Core i7-10700KF 3.8GHz CPU and a GeForce RTX 3070 GPU. The detailed experimental settings are provided below. 

\subsection{Datasets}
To simulate MMO-FL in IoT scenarios, we will use two real-world multi-modal datasets, UCI-HAR and MVSA-Single. Detailed descriptions of both datasets are provided below.

\textbf{UCI-HAR}: The UCI-HAR dataset is a widely recognized resource for human activity recognition research. It includes 10299 data samples collected from 30 participants (average age: 24) engaging in six activities (six classes): walking, walking upstairs, walking downstairs, sitting, standing, and lying down. These activities were recorded using smartphone sensors, specifically accelerometers and gyroscopes, which capture three-dimensional motion data. The sensors sampled data at 50 Hz, producing 128 readings per sensor axis within each time window. This dataset is utilized in our experiments to analyze sensor-based human activity recognition using three-dimensional motion data. 

\textbf{MVSA-Single}: The MVSA-Single dataset is tailored for multimodal sentiment analysis research, focusing on the integration of textual and visual cues from social media. It comprises 5,129 image-text pairs, where each sample consists of a single image accompanied by corresponding textual content. Each pair is annotated with one of three sentiment labels: Positive, Neutral, or Negative, reflecting the emotional tone jointly conveyed by the image and text. 

Both of the original datasets are static and designed for offline learning. To align with the requirements of online learning, they must be transformed into dynamic datasets. The transformation process is described in detail in the following.

\subsection{Online Data Generation}
In the experiment, operating within an online learning scenario requires the training dataset to be dynamic, with data collected at the start of each global round.  To ensure sufficient data samples for good training performance, we collect the initial dataset at the beginning of the training process. Considering the differences in dataset types and sizes, distinct online data generation details are used for the UCI-HAR and MVSA-Single datasets.

\textbf{UCI-HAR}: For the UCI-HAR dataset, training involves a total of five clients. Initially, each client is assigned 2000 data samples drawn according to a Dirichlet distribution with a non-IID degree $\alpha$, representing the client's long-term data source. In subsequent global rounds, each client maintains an online local dataset of 500 samples. At every round, 20 new samples are drawn from the long-term data source and appended to the local dataset, while the oldest 20 samples are simultaneously removed. This real-time update mechanism ensures that the local datasets evolve dynamically, satisfying the requirements of online learning.

\textbf{MVSA-Single}: For the MVSA-Single dataset, the training process is distributed across five clients. Initially, each client is allocated 1,500 data samples drawn from a Dirichlet distribution with a non-IID parameter $\alpha$, simulating the client’s long-term data source. In subsequent global rounds, each client maintains an online local dataset consisting of 800 samples. At every round, 20 new samples are drawn from the long-term data source and added to the local dataset, while the oldest 20 samples are removed. This streaming update mechanism ensures that each client’s dataset evolves continuously over time, thereby aligning with the requirements of online learning.

\textbf{Modality Missing Simulation}: To ensure consistency and clearly isolate the impact of missing modalities, we simulate missing modalities by assuming that all clients simultaneously experience the absence of a specific modality within the same global round. This design choice avoids the confounding effects that would arise if different clients experienced different missing modalities across rounds, which would otherwise complicate the evaluation of modality-level influence on learning performance. Given that all datasets used in this study contain only two modalities, we simulate the missing of one modality at a time. Additionally, we define $\lambda$ as the ratio of global rounds in which modality missing occurs to the total number of rounds. In subsequent experiments, we control the value of $\lambda$ to systematically investigate the effect of the frequency of modality missing on learning performance.

\subsection{Model Details}
In the following, we detail the model architectures and key parameters used in our experiments, presented separately for the two datasets.

\textbf{UCI-HAR}: The dataset includes two distinct modalities: accelerometer and gyroscope signals, necessitating the use of modality-specific encoder models. For the accelerometer data, we use a CNN-based model as the encoder. This model consists of five convolutional layers and one fully connected layer. For the gyroscope data, we use an LSTM-based model with one LSTM layer and one fully connected layer. Both encoders generate a 128-dimensional feature representation. The shared header model is implemented using two fully connected layers. The learning rate is 0.1, with a decay factor of 0.95 until it reaches 0.001. 

\textbf{MVSA-Single}: The dataset includes two distinct modalities: text and
image data, also need modality-specific encoder models. For the image data, a four-layer CNN is utilized, with a modified output layer designed to produce a 128-dimensional feature vector. For text data, a two-layer LSTM network is employed, also generating a 128-dimensional output. The modality-specific features are subsequently processed by a shared header model consisting of two fully connected layers. The learning rate is 0.01, with a decay factor of 0.99 until it reaches 0.001.

\subsection{Benchmarks} In our experiments, we employ several benchmarks for performance comparison. Given that the problem of missing modalities in MMO-FL has received limited attention in prior work, established benchmarks are not available for direct comparison. To address this gap, we include both upper and lower bounds of learning performance and introduce a simple baseline that fills the missing modality with zero-padding. Detailed descriptions of the benchmarks are provided as following.

\textbf{Full Modality (FM)}. In this ideal scenario, all clients are presumed to collect data across all modalities through their respective sensors, ensuring that no modalities are missing in any global round.

\textbf{Partial Modality (PM)}. In this scenario, the failure of certain sensors during some global rounds may prevent the collection of specific modal data. Consequently, the client is trained using only the remaining available modalities.

\textbf{Zero Filling (ZF)}. In this scenario, certain sensors may fail during some global rounds, resulting in the inability to collect specific modal data. To address this issue, a straightforward baseline approach of zero-filling is employed to substitute for the missing modalities with zero-valued representations.

Based on the simulation setup described above, we present the following experimental results.

\subsection{Simulation Results}
In this section, we present the experimental results of MMO-FL. We begin with a comparative analysis of the proposed PMM algorithm against several benchmarks, evaluating overall learning performance. To further validate the effectiveness of PMM, we conduct four ablation studies. First, we investigate the effect of varying the modality missing rate on learning performance. Second, we study the effect of different Non-IID level setting. Third, we evaluate the performance implications and communication efficiency gained through the use of quantized prototype uploads. Final, we examine the impact of the delayed update strategy, assessing both its influence on performance and its potential for reducing communication and computational overhead. All simulation results are averaged over 10 random runs.

\textbf{Performance Comparison}. We begin by evaluating the learning performance of the proposed PMM algorithm compared to benchmarks in the presence of missing modalities within the MMO-FL scenario. The test accuracy comparison between PMM and benchmarks are shown in Fig.~\ref{mmofl-pc}(a) and Fig.~\ref{mmofl-pc}(b) based on UCI-HAR dataset with configuration $[\lambda = 0.5, \alpha = 10]$ and MVSA-Single dataset with $[\lambda = 0.5, \alpha = 1]$, respectively. Based on the above figures, we get several key observations. First, we observe that PMM significantly outperforms both the PM and ZF benchmarks in terms of test accuracy. Remarkably, PMM even achieves performance superior to the FM setting in both cases, indicating that the prototypes generated by PMM not only compensate for the missing modalities but also serve as robust substitutes that enhance overall learning. Second, we observe that the ZF yields better performance than the PM case, suggesting that explicitly preserving the structure of the missing modality even with zero-filled inputs can be more beneficial than entirely omitting the modality during training. Third, we observe that the performance of PMM improves over time, initially lagging behind the FM but eventually surpassing it in the later stages. This behavior highlights the effectiveness of the PMM, which incrementally refines prototype representations over time. As the training progresses, these prototypes become increasingly representative of each class and modality, thus outperforming the raw modality data in guiding the learning process.

\begin{figure}[h]
\centering
\vspace{-15pt}
\subfloat[Test Accuracy with UCI-HAR ]{\includegraphics[width=0.49\linewidth]{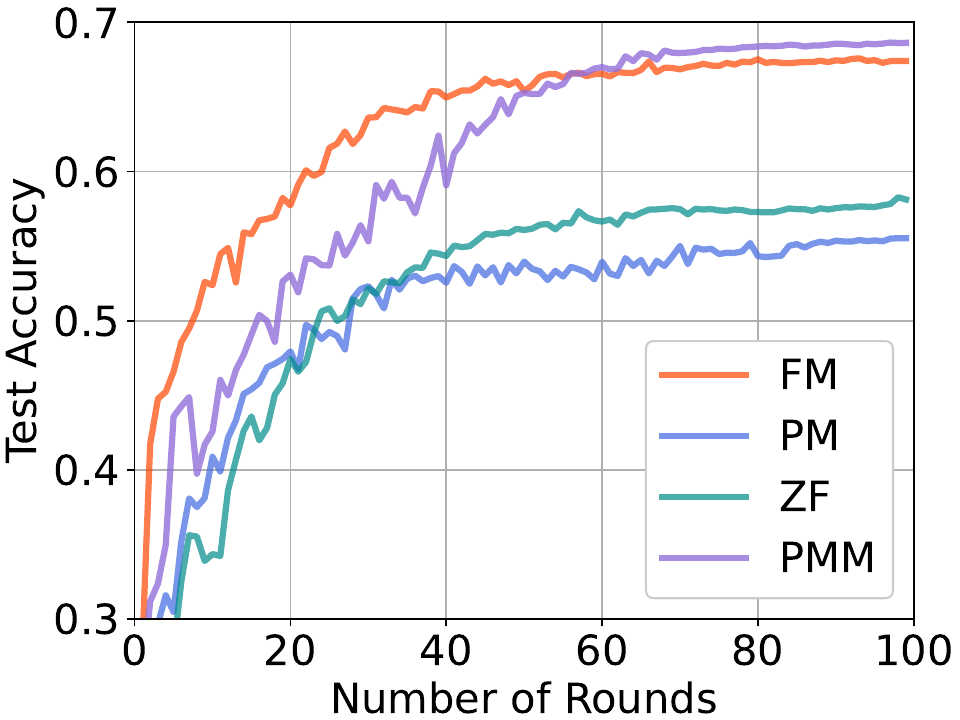}} 
\subfloat[Test Accuracy with MVSA-Single ]{\includegraphics[width=0.49\linewidth]{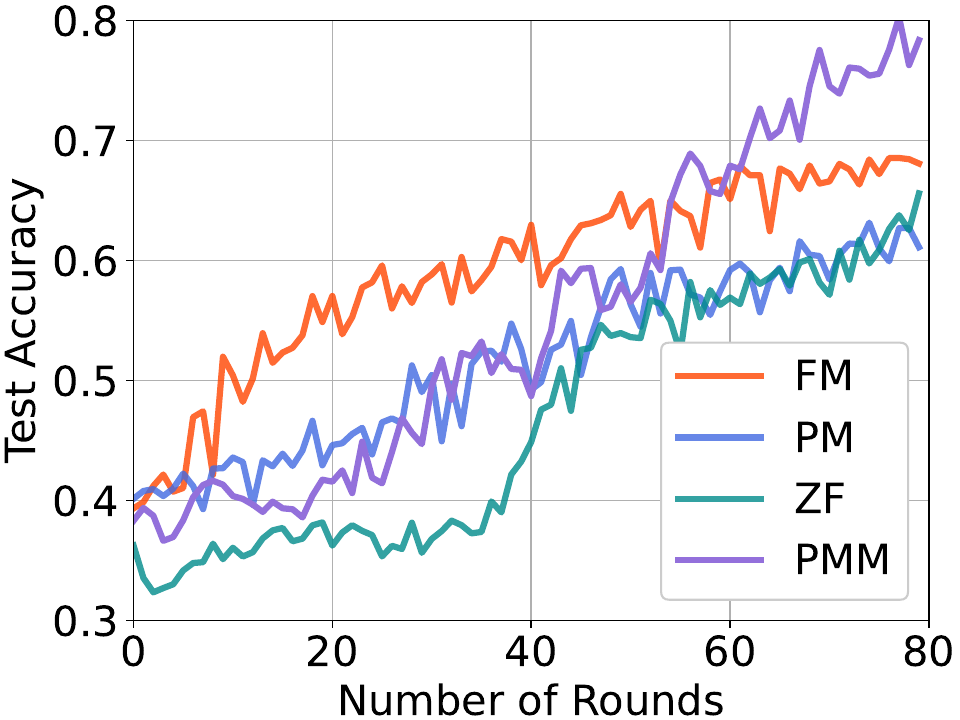}} 
\caption{Performance comparison of proposed algorithm and benchmarks with modality missing.}   \label{mmofl-pc}
\vspace{-10pt}
\end{figure}

Building on the performance comparison results presented above, which demonstrate the effectiveness of the proposed PMM algorithm in the MMO-FL setting, we now proceed with a series of ablation studies to analyze the specific impact of key parameters on learning performance.

\textbf{Impact of Modality Missing Rate}. In this part, we will explore the effects with the modality missing rate $\lambda$ on learning performance. The results for the UCI-HAR and MVSA-Single datasets are presented in Fig.~\ref{mmofl-mmr}(a) and Fig.~\ref{mmofl-mmr}(b), respectively. Several key observations can be made: First, a smaller value of $\lambda$ consistently leads to better learning performance. This improvement is attributed to two factors. On one hand, a lower missing rate results in more rounds where the persistent global prototypes can be accurately updated, allowing them to better represent the underlying modality distributions. On the other hand, fewer occurrences of missing modalities naturally reduce information loss during training, thereby enhancing performance. Second, when $\lambda = 0.7$, corresponding to a high frequency of missing modalities, we observe a significant performance drop in both datasets. This indicates that the frequency of modality absence plays a critical role in determining model effectiveness, and excessive modality missing can severely degrade learning performance.

\begin{figure}[h]
\centering
\vspace{-15pt}
\subfloat[Test Accuracy with UCI-HAR ]{\includegraphics[width=0.49\linewidth]{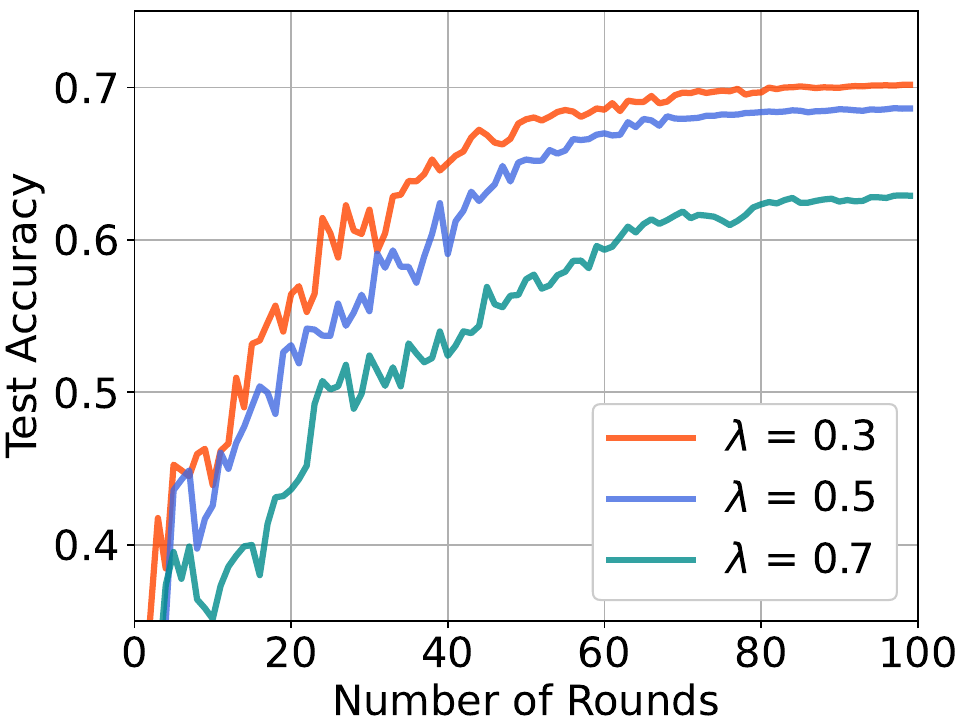}} 
\subfloat[Test Accuracy with MVSA-Single ]{\includegraphics[width=0.49\linewidth]{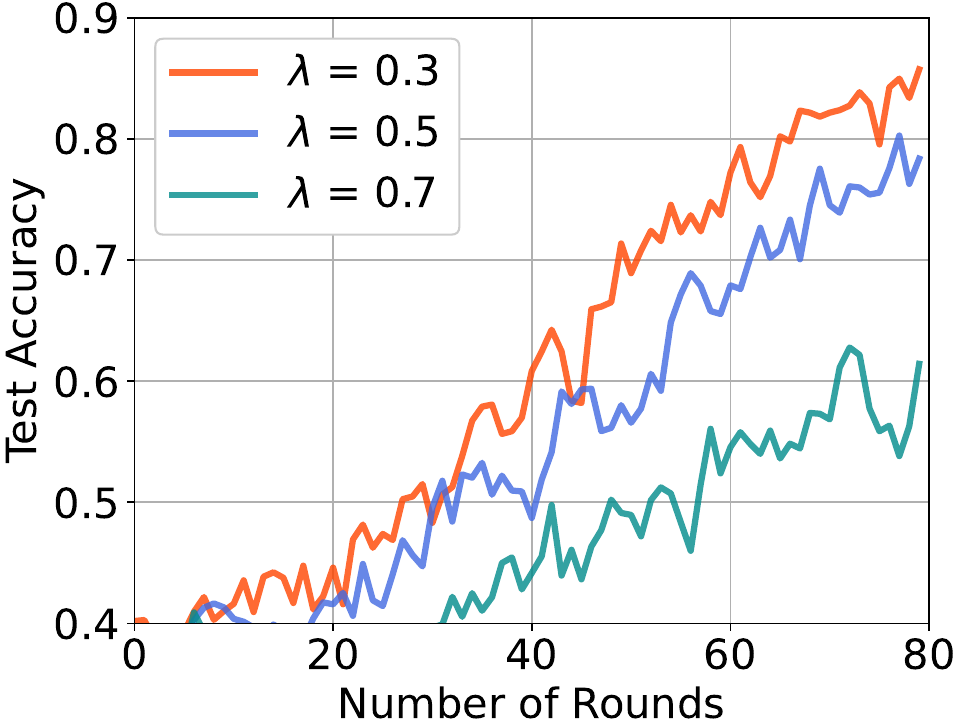}} 
\caption{Performance comparison of proposed algorithm with different modality missing rate.}   \label{mmofl-mmr}
\vspace{-10pt}
\end{figure}

\textbf{Impact of Non-IID Level}. In this part, we will explore the effects with the Non-IID level $\alpha$ on learning performance. The corresponding are illustrated in Fig.~\ref{mmofl-dc}(a) and Fig.~\ref{mmofl-dc}(b), respectively. Here, a larger value of $\alpha$ indicates a lower degree of data heterogeneity. As expected, we observe that increasing $\alpha$ leads to improved learning performance in both datasets, which aligns with the well-established understanding of Non-IID effects in conventional FL. Interestingly, we also observe that the impact of $\alpha$ on performance is less pronounced than that of the modality missing rate $\lambda$. This suggests that the use of prototype-based compensation in PMM may help alleviate the adverse effects of data heterogeneity, thereby maintaining robust performance across varying degrees of Non-IID.

\begin{figure}[h]
\centering
\vspace{-15pt}
\subfloat[Test Accuracy with UCI-HAR ]{\includegraphics[width=0.49\linewidth]{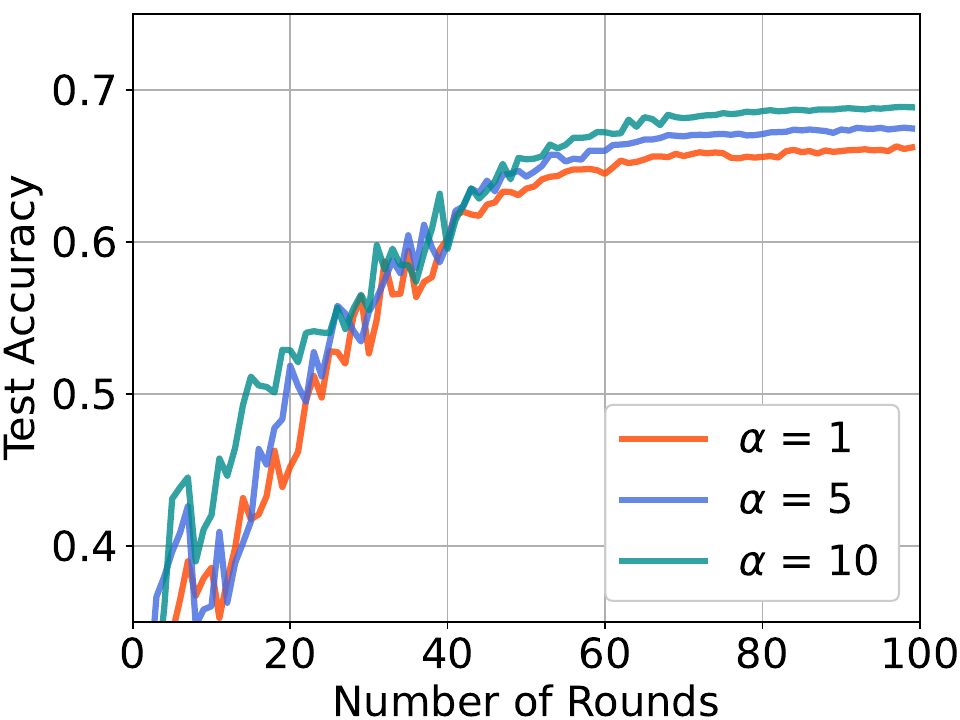}} 
\subfloat[Test Accuracy with MVSA-Single ]{\includegraphics[width=0.49\linewidth]{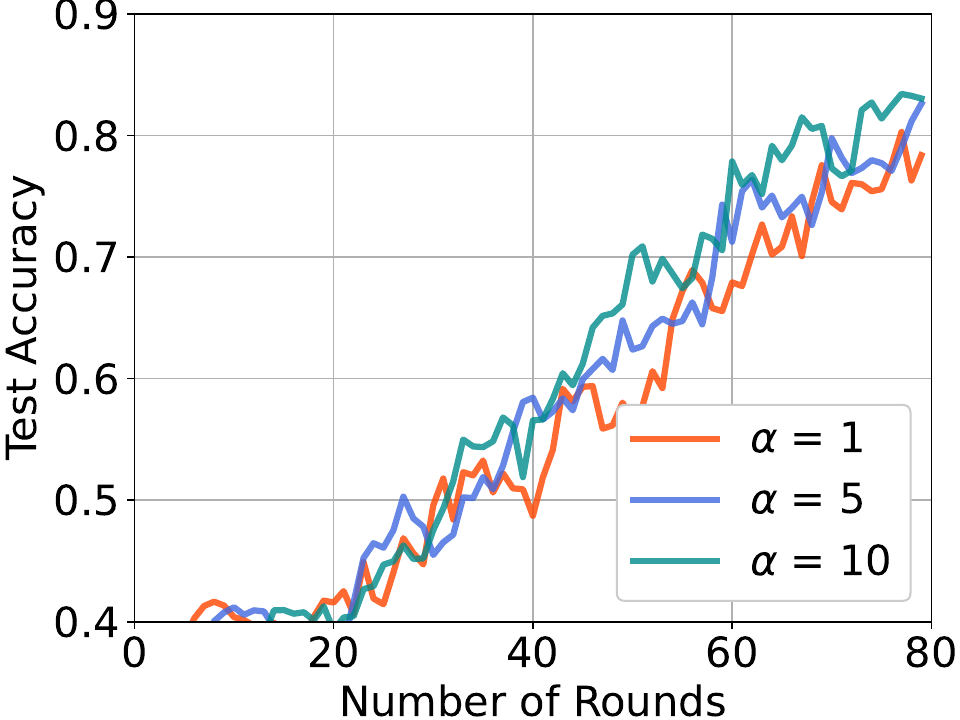}} 
\caption{Performance comparison of proposed algorithm with different Non-IID level.}   \label{mmofl-dc}
\vspace{-10pt}
\end{figure}

Next, we evaluate the efficiency of the proposed PMM algorithm by examining the quantized upload and the delayed update strategies, which were previously introduced in the remarks of the PMM algorithm section. 

\textbf{Impact of Quantized Upload}. In this part, we investigate the impact of applying quantization to local prototype during the PMM upload process as a means to further reduce communication overhead. Specifically, we adopt an uniform scalar quantizer. The parameter $b$ denotes the number of bits per component used for compression, when quantization is not applied, $b = 32$, corresponding to full-precision transmission. Applying quantization with $b$ bits yields a total of $2^b$ quantization levels. We evaluate the effect of varying $b \in [2, 4]$ on learning performance. The results, as illustrated in Fig.~\ref{mmofl-quanti}(a) and Fig.~\ref{mmofl-quanti}(b) by round, and Fig.~\ref{mmofl-quanti}(c) and Fig.~\ref{mmofl-quanti}(d) by communication cost, reveal the trade-off between model performance and communication efficiency that with the effect of quantization. As expected, larger values of $b$ lead to improved performance but incur higher communication costs. Importantly, we observe that the prototypes do not require high precision to be effective. As a result, quantization introduces only a minor reduction in learning performance, while significantly lowering communication overhead. This demonstrates that PMM can be efficiently integrated into the MMO-FL framework without introducing substantial communication burden.

\begin{figure}[h]
\centering
\vspace{-15pt}
\subfloat[Test Accuracy with UCI-HAR ]{\includegraphics[width=0.49\linewidth]{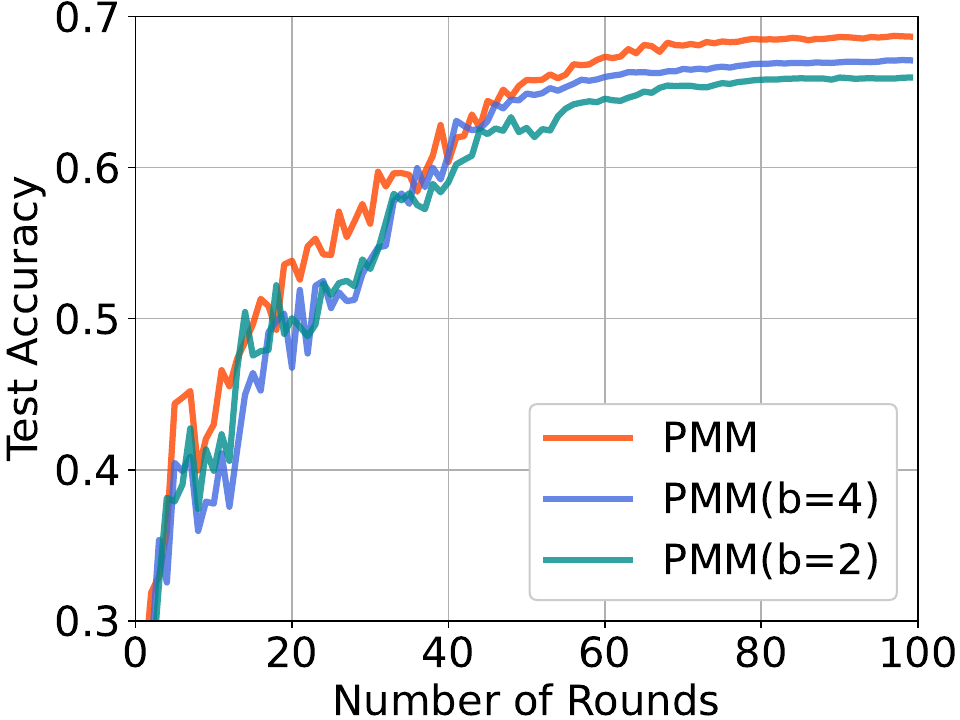}} 
\subfloat[Test Accuracy with MVSA-Single ]{\includegraphics[width=0.49\linewidth]{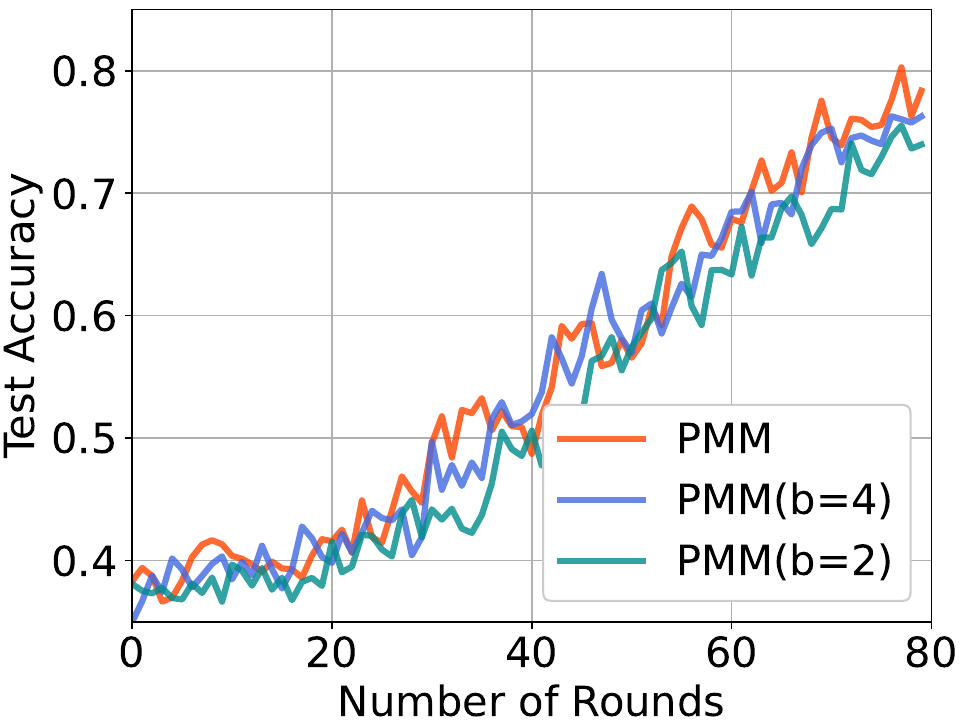}} \\
\subfloat[Test Accuracy with UCI-HAR by communication cost for PMM]{\includegraphics[width=0.49\linewidth]{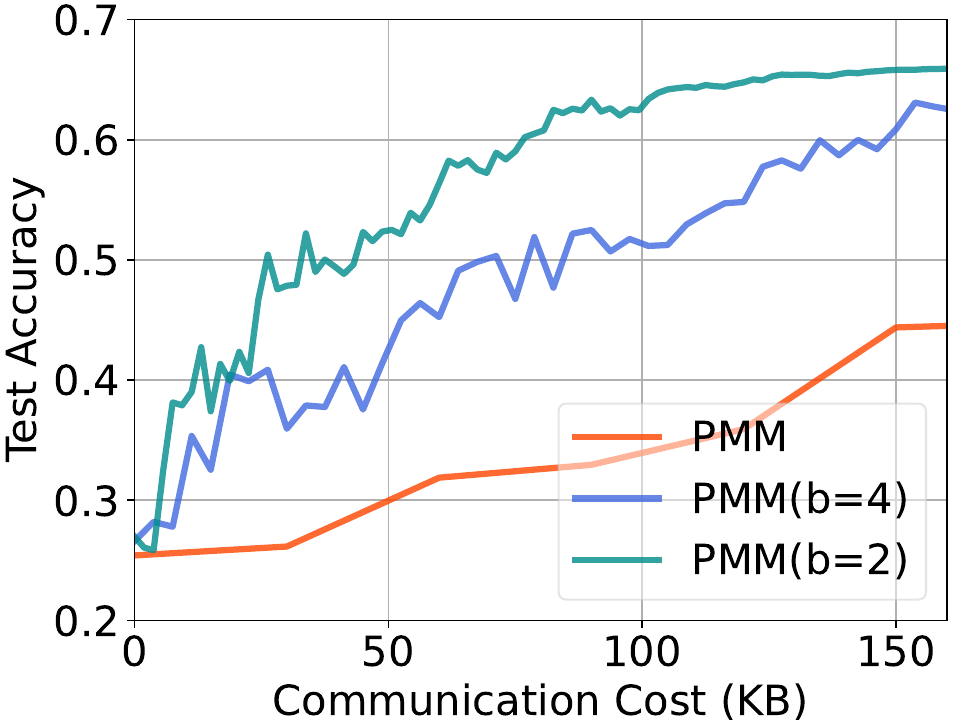}} 
\subfloat[Test Accuracy with MVSA-Single by communication cost for PMM]{\includegraphics[width=0.49\linewidth]{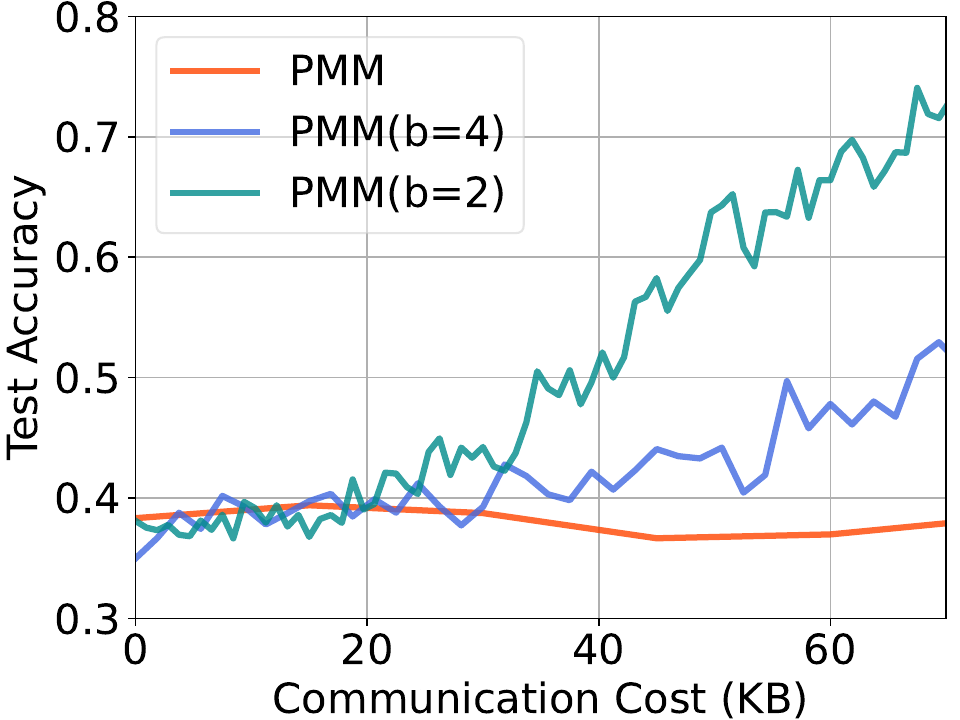}}
\caption{Performance comparison of proposed algorithm with different level of quantization.}   \label{mmofl-quanti}
\vspace{-15pt}
\end{figure}

\textbf{Impact of Delayed Update}. In this section, we evaluate the impact of applying delayed updates to the PMM algorithm. We investigate the impact of different update intervals by varying the frequency of updates with interval values of $[0, 2, 4]$, where a modality must accumulate the corresponding number of occurrences before triggering an update. The experimental results based on the UCI-HAR and MVSA-Single datasets are shown in Fig.~\ref{mmofl-delay}(a) and Fig.~\ref{mmofl-delay}(b). Our findings show that while delayed updates lead to a moderate decrease in performance, the PMM algorithm still maintains satisfactory learning outcomes. Additionally, In Fig.~\ref{mmofl-delay}(c) and Fig.~\ref{mmofl-delay}(d) we report the corresponding computation cost, which closely aligns with communication cost since communication is only triggered after computation. These results highlight an inherent trade-off: updating at every round yields the best performance but incurs higher resource consumption, whereas adopting delayed updates reduce computational and communication overhead at the cost of slight performance degradation. This makes delayed updating a compelling option in resource-constrained environments.

\begin{figure}[h]
\centering
\vspace{-15pt}
\subfloat[Test Accuracy with UCI-HAR]{\includegraphics[width=0.49\linewidth]{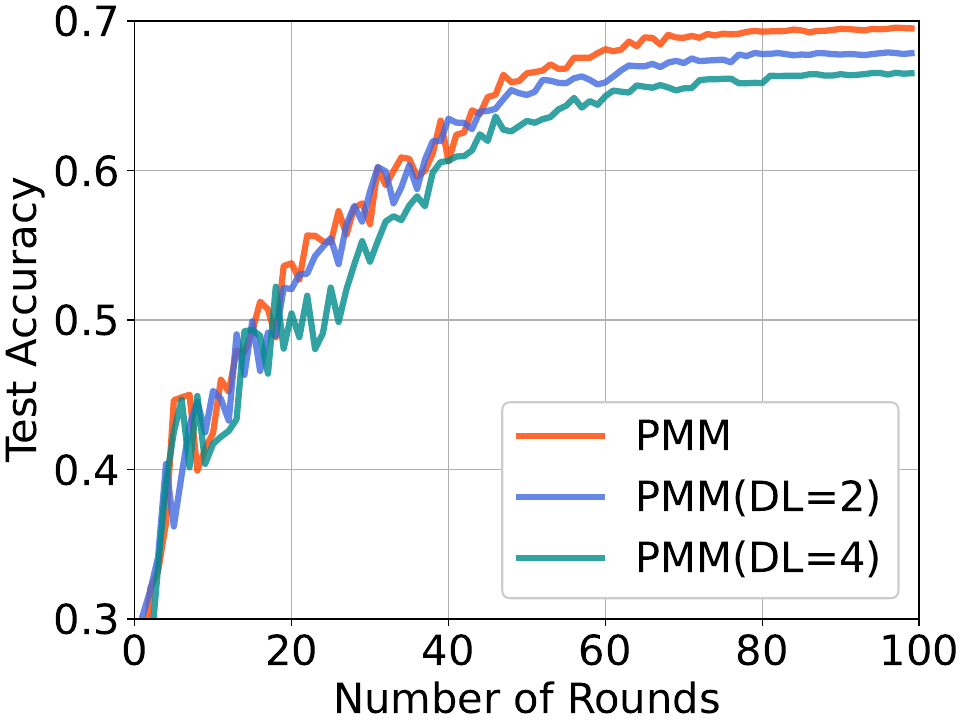}} 
\subfloat[Test Accuracy with MVSA-Single]{\includegraphics[width=0.49\linewidth]{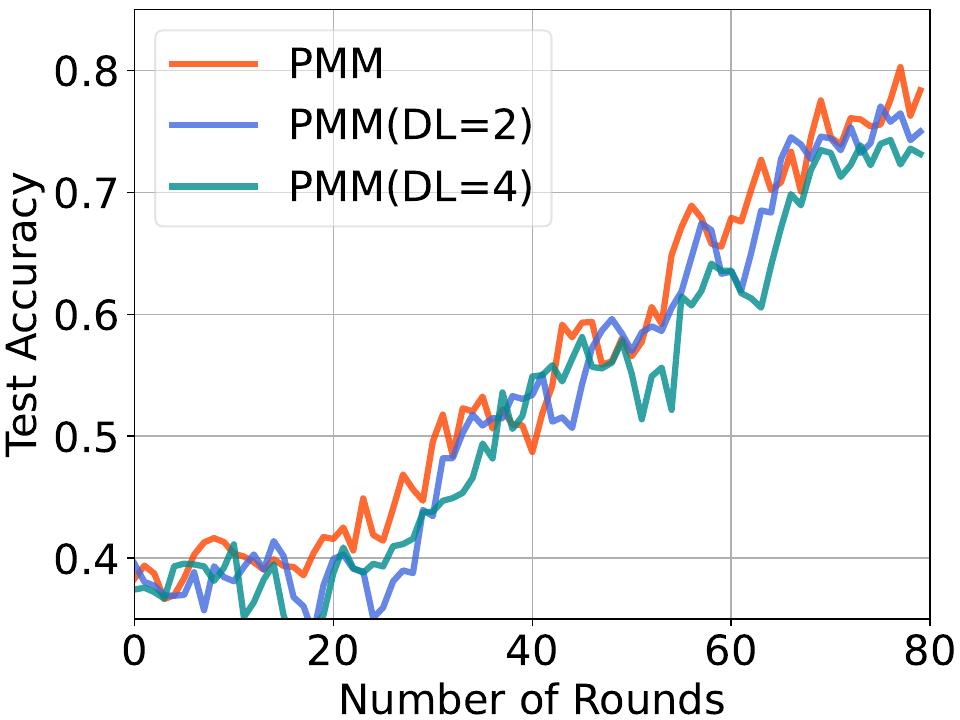}} \\
\subfloat[Test Accuracy with UCI-HAR by computation times for PMM]{\includegraphics[width=0.49\linewidth]{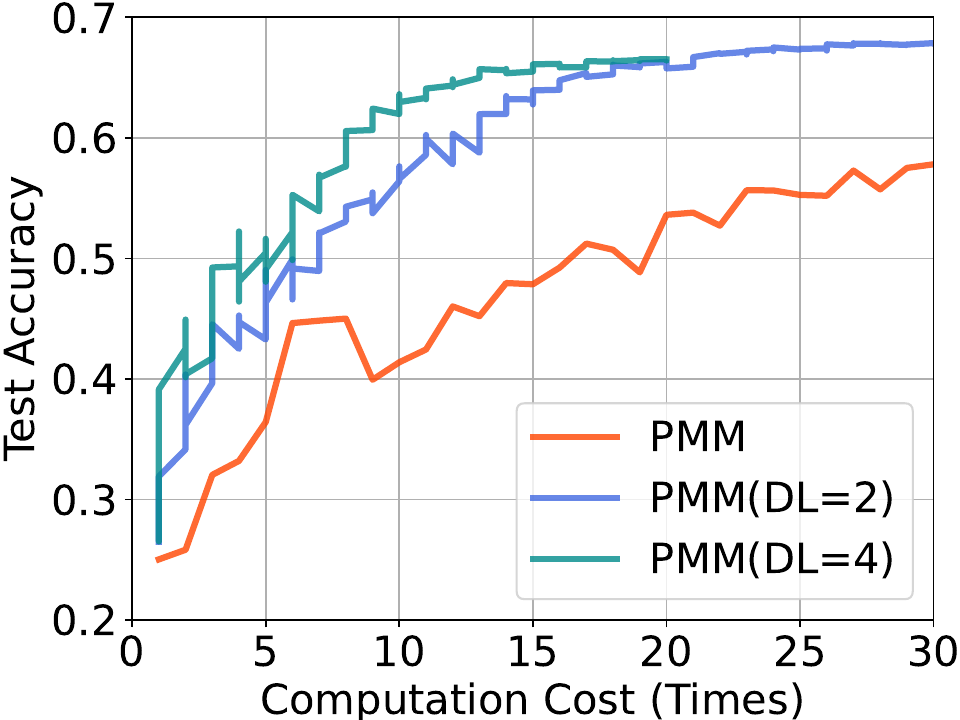}} 
\subfloat[Test Accuracy with MVSA-Single by computation times for PMM]{\includegraphics[width=0.49\linewidth]{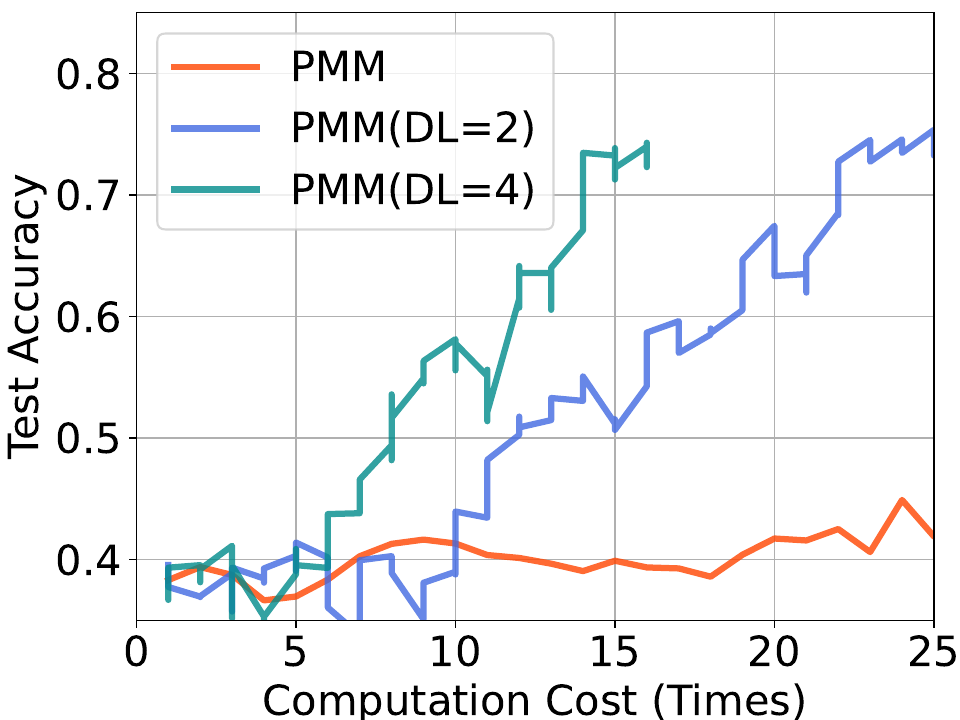}} 
\caption{Performance comparison of proposed algorithm with different level of delay.}   \label{mmofl-delay}
\vspace{-10pt}
\end{figure}

Through the series of experiments presented above, we have thoroughly evaluated the performance of the proposed PMM algorithm in comparison to benchmarks, as well as examined its scalability across various settings.

\section{Conclusion }
In this work, we propose the MMO-FL framework to address the challenges of distributed cooperative learning over real-time multimodal data in IoT environments with edge intelligence devices. Recognizing the instability of modality-specific sensors during continuous data collection, we explicitly consider and tackle the problem of modality missing throughout the learning process. To mitigate its impact, we introduce the PMM algorithm, which leverages prototype learning to effectively reconstruct and compensate for missing modal information. The proposed approach is supported by rigorous theoretical analysis and validated through extensive experiments. In future work, we aim to integrate real-time IoT multimodal data and develop practical testbeds to evaluate and enhance the MMO-FL framework. These efforts will further facilitate its deployment in real-world IoT applications.

\bibliographystyle{IEEEtran}
\bibliography{bibligraphy}

\end{document}